\documentclass{article}

    \PassOptionsToPackage{numbers, compress}{natbib}
\usepackage[final]{neurips_wrl2022}

\usepackage[utf8]{inputenc} % allow utf-8 input
\usepackage[T1]{fontenc}    % use 8-bit T1 fonts
\usepackage{hyperref}       % hyperlinks
\usepackage{url}            % simple URL typesetting
\usepackage{booktabs}       % professional-quality tables
\usepackage{amsfonts}       % blackboard math symbols
\usepackage{nicefrac}       % compact symbols for 1/2, etc.
\usepackage{microtype}      % microtypography
\usepackage{graphicx}
\usepackage[table,xcdraw]{xcolor}
\usepackage{graphics} % for pdf, bitmapped graphics files
\usepackage{epsfig} % for postscript graphics files
\usepackage{times} % assumes new font selection scheme installed
\usepackage{amsmath} % assumes amsmath package installed
\usepackage{amssymb}  % assumes amsmath package installed
\usepackage{adjustbox}
\usepackage{mathrsfs}   
\usepackage{bm} 
\usepackage{caption}
\usepackage{subcaption}% Needed to meet printer requirements.
\usepackage{color}
\usepackage[ruled,lined, linesnumbered]{algorithm2e}
\makeatletter 
\g@addto@macro{\@algocf@init}{\SetKwInOut{Parameter}{Parameters}} 
\makeatother

\usepackage{amsthm}
\usepackage{svg}
\usepackage{wrapfig}
\usepackage{enumitem}
\usepackage{todonotes}
\usepackage{makecell} 
\usepackage{mathtools}
\usepackage{wrapfig}
\usepackage{enumitem}
\usepackage{todonotes}
\usepackage{makecell} 
\usepackage{mathtools}
\usepackage{multicol}

\usepackage{tablefootnote}

\newcounter{matriz}

\newcounter{tableeqn}[table]
\renewcommand{\thetableeqn}{\arabic{tableeqn}}
\newcounter{tablesubeqn}[tableeqn]

\usepackage{etoolbox}% http://ctan.org/pkg/etoolbox
\AtBeginEnvironment{align}{\setcounter{subeqn}{0}}% Reset subequation number at start of align
\newcounter{subeqn} %

\newtheorem{theorem}{Theorem}

\newtheorem{corollary}{Corollary}

\newcounter{daggerfootnote}

\title{Certifiably-correct Control Policies for Safe Learning and Adaptation in Assistive Robotics}

\author{
Keyvan Majd$^{1\dagger}$, Geoffrey Clark$^{1}$, Tanmay Khandait$^{1}$, Siyu Zhou$^{1}$,\AND Sriram Sankaranarayanan$^{2}$, Georgios Fainekos$^{3}$, Heni Ben Amor$^{1}$ \vspace{8 pt}\\
$^{1}$Arizona State University, $^{2}$University of Colorado Boulder, $^{3}$ Toyota NA-R\&D\vspace{8 pt}\\
$^{\dagger}\texttt{majd@asu.edu}$
}

\begin{document}

\maketitle

\begin{abstract}
% This is a \textbf{test upload} from the organizing committee. 
% No need to review.

  Guaranteeing safety in human-centric applications is critical in robot learning 
  as the learned policies may demonstrate unsafe behaviors in formerly unseen scenarios. 
  We present a framework to locally repair an erroneous policy network 
  to satisfy a set of formal safety constraints using Mixed Integer Quadratic Programming (MIQP). 
  Our MIQP formulation explicitly imposes the safety constraints to the learned policy
  while minimizing the original loss function. 
  The policy network is then verified to be locally safe. 
  We demonstrate the application of our framework to derive safe policies for a robotic lower-leg prosthesis.
\end{abstract}

\begin{figure}[h!]
    \centering    \includegraphics[width=0.98\textwidth]{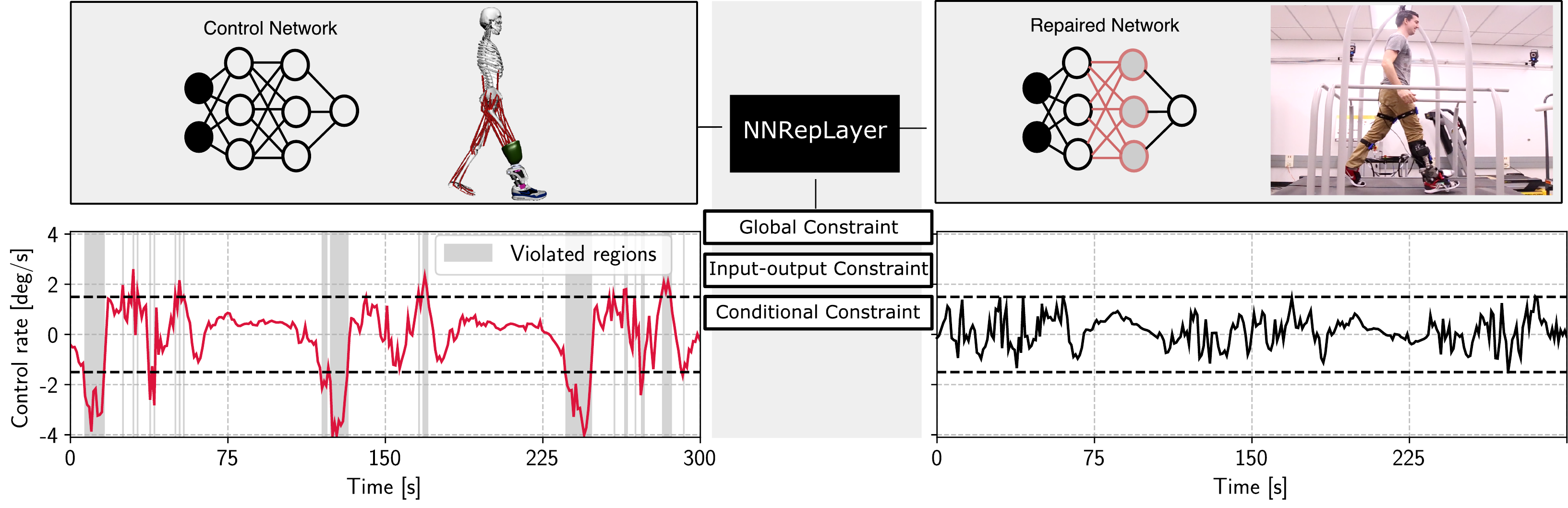}
    \caption{Left: A trained neural-network policy to control a prosthesis violates formal safety constraints. Right: Our framework repairs the violation while maintaining the underlying behavior.}
    \label{fig:repair_intro}
    \vspace{-5pt}
\end{figure}

\section{Introduction}
Deep network policies can help capturing the complicated human-robot interaction dynamics and adapting control parameters automatically to the user's individual characteristics in the assistive robotic devices.
Despite this advanced capability, deep neural network (DNN) policies are not widely used in the field of assistive robotics 
since the trained policies may generate unsafe behaviors when encountering unseen inputs.
% Due to the safety critical nature of the human-assistive control tasks, 
% the trained policies should strictly comply with the safety protocols and constraints [cite]. 
In case of lower-leg prosthesis, for example, control values and joint angles should not exceed certain limits.
%===============================================================================
% lit. review:
One approach to ensure that a NN satisfies a given set of safety properties is to use retraining and fine-tuning  based on counter-examples~\citep{sinitsin2019editable,ren2020few, DongEtAl2021qrs, taormina2020performance, yang2022neural}. 
However, this approach has a number of pitfalls. First of all, in both retraining and fine-tuning the labels of the desired data, i.e., the data that satisfy the constraints, are not known. 
Most critically, retraining and fine-tuning are typically based on gradient descent optimization methods, so they cannot guarantee that the result satisfies the provided constraints.
The methods presented in \citep{FuLi2022iclr} and \citep{sotoudeh2021provable} rely on extending the trained policy architecture or producing decoupled networks, respectively, to modify the output of network only in the faulty linear regions of input space. 
The method proposed in \citep{FuLi2022iclr} is only suitable to the classification tasks as it only satisfies the constraints for the faulty samples and does not consider the minimization of original loss function. 
It is also unclear how the approach scales for high dimensional
inputs, since it requires partitioning the input space into affine subregions.
The decoupling technique proposed in \citep{sotoudeh2021provable} causes the repaired network to be discontinuous,
so it cannot be employed in robot learning and control.
The application of \citep{sotoudeh2021provable} is also limited to the policies with less than three inputs.
\begin{wrapfigure}{r}{0.21\textwidth}
\vspace{-7pt}
    \centering
    \includegraphics[width=0.19\textwidth]{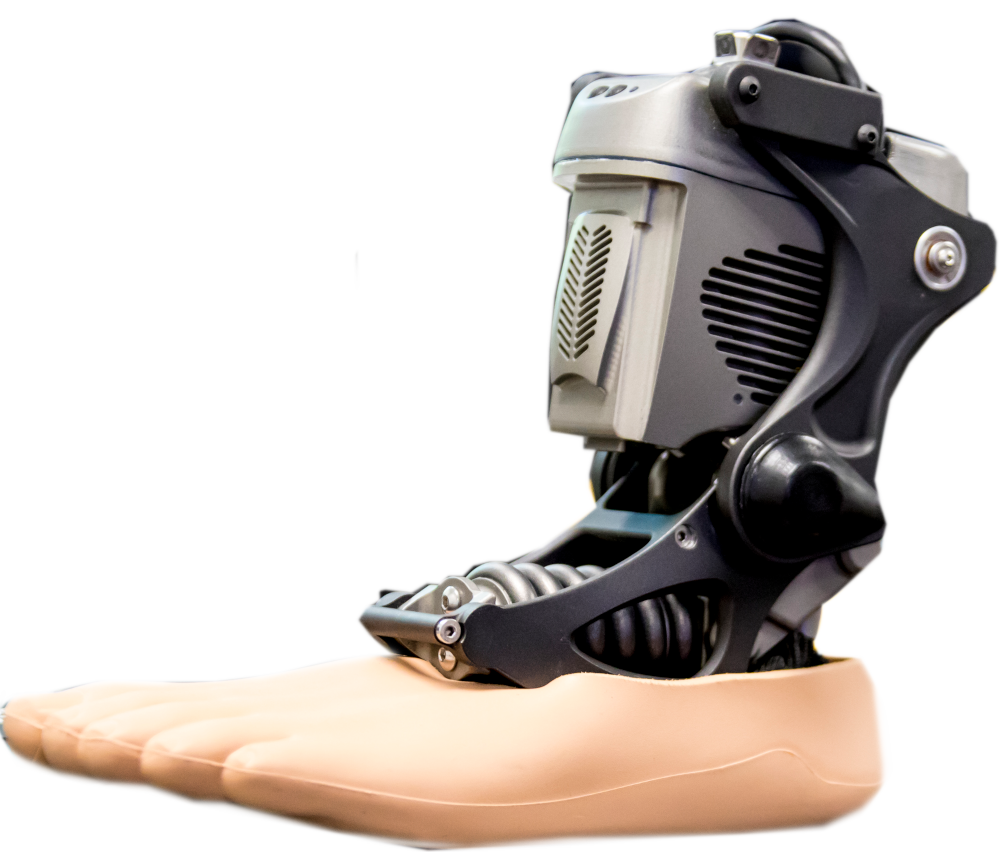}
    \caption{Lower-leg prosthesis }
    \label{fig:prosthesis}
    \vspace{-10pt}
\end{wrapfigure}
Finally, \citep{goldberger2020minimal} only repairs the weights of final layer
which drastically reduces the space of possible successful repairs (and mostly a repair is not
even feasible).

%===============================================================================

We introduce NNRepLayer, a framework to certifiably repair a network policy to satisfy a set of given safety constraints with minimal deviation from the performance of original trained network.
We particularly used our approach to derive controllers for a robotic lower-leg prosthesis that satisfy basic safety conditions, see Fig. \ref{fig:repair_intro}.
Given a set of safety constraints on the output of the trained policy for a set of input samples, 
NNRepLayer formulates a Mixed-integer Quadratic Programming (MIQP) to modify the weights of policy in a layer-wise fashion subject to the desired safety constraints.
Applying NNRepLayer formally guarantees the satisfaction of constraints for the given input samples.
Moreover, we propose an algorithm to employ NNRepLayer and a sound verifier in the loop that improves the adversarial accuracy of the trained policy 
(satisfaction of constraints for the inputs in a norm-bounded distance of the repaired samples).
% Finally, in real-world robot experiments
% we show that the introduced methodology produces safe neural policies for a lower-leg prosthesis
% satisfying a variety of constraints. 
Note that the repair problem is fundamentally different from verification. 
Verification \citep{TjengXT2019iclr, liu2021algorithms} explores the input space of network to obtain tight bounds over the output
while repair optimizes the NN parameters to ensure the satisfaction of constraints by the network’s output.

{\bf Notation.} 
We denote the set of variables $\{a_1,a_2,\cdots,a_N\}$ with $\{a_n\}^N_{n=1}$.
Let $\pi_{\theta}$ be a network policy with $L$ hidden layers. 
The nodes at each layer $l \in \{l\}^L_{l=0}$ are represented by $x^l$, 
where $ \lvert x^l \rvert$ denotes the dimension of layer $l$ ($x^0$ represents the input of network). 
The network's output $\pi_{\theta}(x^0)$ is denoted by $y$.
We consider fully connected policy networks with weight and bias terms $\{(\theta^l_w,\theta^l_b)\}^{L+1}_{l=1}$.
% Each two subsequent layers $l-1$ and $l$ are fully connected with weight and bias terms $\theta^l_w$ and $\theta^l_b$, respectively. 
The training data set of $N$ inputs $x^0_n$ and target outputs $t_n$ is denoted by $\{(x^0_n,t_n)\}^N_{n=1}$ 
sampled from the input-output space
$\mathcal{X}\times\mathcal{T}\subseteq\mathbb{R}^{\lvert x^0\rvert}\times\mathbb{R}^{ \lvert t\rvert}$.
The vector of nodes at layer $l$ for sample $n$ is denoted by $x^l_n$.
% The value of each hidden node is calculated using the weighted sum of nodes in 
% its previous layer passed through a nonlinear activation function. 
In this work, we focus on the policy networks with the Rectified Linear Unit (ReLU) activation function $R(z) = \max\{0,z\}$. 
Thus, given the $n$\textsuperscript{th} sample, in the $l$\textsuperscript{th} hidden layer, 
we have $x^l = R\left(\theta_w^lx^{l-1}+\theta_b^l\right)$. 
The last layer is also represented as $y = \theta_w^{L+1}x^{L}+\theta_b^{L+1}$.
% Finally, $N\!N^{\theta_w^l ,b^l}$ denotes the change of $l$\textsuperscript{th} layer's weights and bias terms in $N\!N$ by $\theta_w^l$ and $b^l$, respectively.
% \vspace{-1pt}
\section{NNRepLayer}
Without loss of generality, we motivate and discuss our approach using the task of learning safe robot controllers for a lower-leg prosthesis, see Fig. \ref{fig:prosthesis}. The goal is to learn a policy $\pi_\theta$ which generates control values for the ankle angle of the powered prosthesis given a set of sensor values. 
% Most critically, however, policy $\pi_\theta$ is required to satisfy a set of safety constraints $\Psi$. 
% Fig.~\ref{fig:repair_intro} provides an overview of our methodology in addressing this challenge. 
% For now, we assume that an unsafe prior policy network may exist, see Fig.~\ref{fig:repair_intro}~(left). 
% The network parameters $\theta$ may be learned with through imitation learning~\citep{gao2020recurrent}, reinforcement learning~\citep{gao2020} or any of the common machine learning paradigms. 
The given policy may be optimized for task efficiency, e.g., stable and low-effort walking gaits, but may not yet satisfy any safety constraints $\Psi$. 
Our goal is to find an adjusted set of network parameters that satisfy any such constraint. 
% We may now choose, for example, to restrict the control rate to specific bounds. 
% Running the network on the input values of a validation data set reveals that violations occur at a number of time steps, 
% as seen in Fig.~\ref{fig:repair_intro}. 
% Our approach, called \textbf{NNRepLayer} (Layer-wise Neural Network Repair), takes the original network parameters $\theta$ and the predicates $\Psi$ and yields the updated parameters that generate no violations. 
%===============================================================================
% We formulate our framework as the minimization of the loss function $E(\theta_w,\theta_b)$ 
% subject to $(x^0,t)\in\mathcal{X}\times\mathcal{T}$ and $\Psi(y,x^0)$ for the inputs of interest $x^0\in \mathcal{X}_{r}$. 
% However, the resulting optimization problem is non-convex and difficult to solve due to the nonlinear ReLU activation functions and the high-order nonlinear constraints resulted from the multiplication of terms involving the weight/bias variables. 
% In our approach, we obtain a sub-optimal solution by just focusing on repairing a single layer. 
% We therefore modify the weight and bias terms of a single layer to adjust the predictions so as to minimize $E(\cdot)$ and to satisfy $\Psi(\cdot)$. 
% Thus, we solve the following problem,
We formulate our framework as the minimization of the loss function $E(\theta_w,\theta_b)$ 
subject to $(x^0,t)\in\mathcal{X}\times\mathcal{T}$ and $\Psi(y,x^0)$ for the inputs of interest $x^0\in \mathcal{X}_{r}$. 
However, the resulting optimization is non-convex and difficult to solve due to the nonlinear forward pass of ReLU networks.
Hence, we obtain a sub-optimal solution by just modifying the weights and biases of a single layer to adjust the predictions so as to minimize $E(\cdot)$ and to satisfy $\Psi(\cdot)$. 
The problem is defined as follows,

{\bf Problem Statement (Repair Problem).} \textit{Let $\pi_{\theta}$ denote a trained policy with 
$L$ hidden layers over the training input-output space 
$\mathcal{X}\times\mathcal{T}\subseteq\mathbb{R}^{\lvert x^0\rvert}\times\mathbb{R}^{\lvert t\rvert}$ 
and $\Psi(y,x^0)$ denote a predicate representing constraints on the output $y$ of $\pi_{\theta}$
for the set of inputs of interest $x^0\in \mathcal{X}_r\subseteq\mathcal{X}$. 
NNRepLayer modifies the weights of a layer $l\in\{1,\cdots,L+1\}$ in $\pi_{\theta}$ such that the new network $\pi_{\theta_r}$ satisfies $\Psi(y,x^0)$ while minimizing the loss of network $E(\theta_w^l,\theta_b^l)$ with respect to its original training set.}

Since $\mathcal{X}_r$ and $\mathcal{X}$ are not necessarily convex,
we formulate NNRepLayer over a data set
$\{(x^0_n,t_n)\}^N_{n=1}\sim \mathcal{X}\times\mathcal{T} \cup \mathcal{X}_r\times\Tilde{\mathcal{T}}$,
where $\Tilde{{\mathcal{T}}}$ is the set of original target values of inputs in $\mathcal{X}_r$. 
The predicate $\Psi(x^0,y)$ defined over $x^0\in \mathcal{X}_r$ is not necessarily compatible with the target values in $\Tilde{\mathcal{T}}$. 
It means that the predicate may bound the NN output for $\mathcal{X}_r$ input space such 
that not allowing an input $x^0\in \mathcal{X}_r$ to reach its target value in
$\Tilde{\mathcal{T}}$. 
It is a natural constraint in many applications. 
% For instance, due to the safety constraints, 
% we may not allow a NN controller to follow its original control reference for a given unsafe set of input states.
For a given layer $l$, we also define $E(\theta_w^l,\theta_b^l)$ in the form of sum of square loss 
$E(\theta_w^l,\theta_b^l) = \sum^{N}_{n=1}\lVert y_n(x_n^0,\theta_w^l,\theta_b^l)-t_n\rVert^2_2$,
where $\lVert\cdot\rVert_2$ denotes the Euclidean norm. Since we only repair the parameters of target layer $l$, $\{(\theta^i_w,\theta^i_b)\}^{L+1}_{i=l+1}$ are fixed. 
% Here, since we only repair the weight and bias terms of target layer $l$, the loss term $E$ is a function of $\theta_w^l$ and $\theta_b^l$, respectively. 
% Hence, the weight and bias terms of all layers except the target layer $l$ are fixed. 
We define NNRepLayer as follows.

{\bf NNRepLayer Optimization Formulation.} Let $\pi_{\theta}$ be a policy with $L$ hidden layers,
$\Psi(y,x^0)$ be a predicate, 
and $\{(x^0_n,t_n)\}^N_{n=1}$ be an input-output data set sampled from  $(\mathcal{X}\times\mathcal{T}) \cup (\mathcal{X}_r\times\Tilde{\mathcal{T}})$. 
% Also, let  $E(\theta_w^l,\theta_b^l)$ be the loss function defined in (\ref{eq: loss}). 
NNRepLayer minimizes the loss (\ref{eq:cost}) by modifying $\theta^l_w$ and $\theta^l_b$
subject to the constraints (\ref{eq:last_layer})-(\ref{eq:w_bound}).
\begin{figure*}[t]
\begin{minipage}{\dimexpr.5\textwidth-.5\columnsep}
\setlength{\tabcolsep}{7pt} % Default value: 6pt
\renewcommand{\arraystretch}{1.36} % Default value: 1
\aboverulesep=0ex % Solution part 1 of 3
\belowrulesep=0ex % Solution part 1 of 3
\setlength{\tabcolsep}{0.5em}
\begin{tabular}{|lll|}
\toprule
\refstepcounter{tableeqn} (\thetableeqn)
\label{eq:cost}
&
\multicolumn{2}{l|}{$\underset{\substack{\theta_w^l,\theta_b^l,\delta,y_n,\{x_n^i\}_{i=l}^{L},\{\phi_n^i\}_{i=l}^{L}} } {\text{min}}\!\!\!\!E(\theta_w^l,\theta_b^l)+\delta$,}
\\
&
\multicolumn{2}{l|}{s.t.}
\\

\refstepcounter{tableeqn} (\thetableeqn) \label{eq:last_layer}
&
$y_n=\theta_w^{L+1}x^{L}_n+\theta_b^{L+1}$,
 
& 
\\
\refstepcounter{tableeqn} (\thetableeqn) \label{eq:mid_layer}
&
$x^i_n=R(\theta_w^ix^{i-1}_n+\theta_b^i)$,

& 
\text{for } $\{i\}_{i=l}^{L}$
\\
\refstepcounter{tableeqn} (\thetableeqn) \label{eq:predicates}
&
$\Psi(y_n,x^0_n)$,
& \text{for }$x^0_n\in \mathcal{X}_r$
\\
\refstepcounter{tableeqn} (\thetableeqn)\label{eq:w_bound}
&
\multicolumn{2}{l|}{$\delta \geq \lVert \theta_w^l-\theta_w^{l,init}\rVert_{\infty},~\lVert \theta_b^l-\theta_b^{l,init}\rVert_{\infty}$.}  
 \\\bottomrule
\end{tabular}
\end{minipage}\hfill
\begin{minipage}{\dimexpr.5\textwidth-.5\columnsep}
\RestyleAlgo{boxruled}
\setlength{\algomargin}{1.4em}
% \vspace{-12pt}
    \begin{algorithm}[H]
    \DontPrintSemicolon
    % \dontprintsemicolon
        \caption{NNRepLayer \& Verifier }
        \label{alg:verifier}
            \KwIn{$\pi^{o}_{\theta}, \mathcal{X}_{r}, \Psi$}
            \KwOut{ $\pi^{r}_{\theta}$}
            $\pi^{r}_{\theta}\leftarrow \pi^{o}_{\theta}$\;
            \While{$\mathcal{X}_{r}\notin \emptyset$}{
            $\pi^{r}_{\theta}\leftarrow \textsc{NNRepLayer}(\pi^{r}_{\theta}, \mathcal{X}_{r}, \Psi)$\;
            $\mathcal{X}_{r}\leftarrow \textsc{Verifier}(\pi^{r}_{\theta})$\;}
    \end{algorithm}
\end{minipage}
% \vspace{-10pt}
\end{figure*}

Here, constraints (\ref{eq:last_layer}) and (\ref{eq:mid_layer}) represent the linear forward pass of network's last layer and hidden layers starting from the layer $l$, respectively. 
Except the weight and bias terms of the $l$\textsuperscript{th} layer, i.e. $\theta^l_w$ and $\theta^l_b$, 
the weight and bias terms of the subsequent layers $\{(\theta^i_w,\theta^i_b)\}^{L+1}_{i=l+1}$ are fixed.
The sample values of $x_n^{l-1}$ are obtained by the weighted sum of the nodes in its previous layers starting from $x_n^0$ for all $N$ samples $\{n\}^N_{n=1}$.
Each ReLU node $x^l$ is formulated using Big-M formulation
\citep{belotti2011disjunctive,tsay2021partition} by $x^l\geq \theta_w^{l}x^{l-1}_n+\theta_b^{l}$, $x^l\leq
\big(\theta_w^{l}x^{l-1}_n+\theta_b^{l}\big)-lb(1-\phi)$, and $x^l\leq ub~\phi$,
where $x^l \in [0,\infty)$, and $\phi\in\{0,1\}$ determines the activation status of node $x^l$. 
The bounds $lb, ub\in \mathbb{R}$ are known as Big-M coefficients, $\theta_w^{l}x^{l-1}_n+\theta_b^{l}\in[lb,ub]$, 
that need to be as tight as possible to improve the performance of MIQP solver. 
We used Interval Arithmetic (IA) Method \citep{moore2009introduction, TjengXT2019iclr}
to obtain tight bounds for ReLU nodes (read Appx. \ref{supp: IA} for further details on IA). 
% The ReLU activation functions $R(z)=\max\{0,z\}$ are modeled with a disjunctive inequality constraint $[\phi = 0, x\leq 0 ]\vee[\phi = 1, s\leq 0]$. 
% Assume $\phi$ is a Boolean variable and $z = x-s$, where $x,s\geq 0$. 
% We have $x=z$ when $\phi=1$ ($z\geq 0$) and $x=0$ when $\phi=0$ ($z=-s< 0$). Here, $x$ determines the output of ReLU activation function $R(z)$,
% the disjunctive inequalities are then relaxed as mixed integer algebraic equations by the Big-M method \citep{belotti2011disjunctive}. We used Interval Arithmetic (IA) Method \citep{moore2009introduction, TjengXT2019iclr} to obtain tight bounds for ReLU nodes (read the supplementary materials, Sec. \ref{supp: IA}, for further details on IA). 
% \begin{align}\refstepcounter{tableeqn}\label{eq: disjunction}\tag{\thetableeqn}
%     \left[\begin{array}{c}
%          \phi = 0\\
%           x\leq 0
%     \end{array}\right]\bigvee\left[\begin{array}{c}
%          \phi = 1\\
%           s\leq 0
%     \end{array}\right]. 
% \end{align}
Constraint (\ref{eq:predicates}) is a given predicate on $y$ defined over
$x^0\in\mathcal{X}_r$. NNReplayer addresses the predicates of the form $\bigvee_{c=1}^{C}\psi_c(x^0,y)$ where $C$ represents the number of disjunctive propositions and $\psi_i$ is an affine function of $x^0$ and $y$.
Finally, constraint (\ref{eq:w_bound}) bounds the entry-wise max-norm error between the weight and bias terms $\theta^l_w$ and $\theta^l_b$, 
and the original $\theta_w^{l,init}$ and $\theta_b^{l,init}$ by $\delta$.
%  By adding $\delta$ to (\ref{eq:cost}), we aim to minimize weights deviation as well as $E(\theta_w^l,\theta_b^l)$ in the repair process. 
% We also assume that the predicate $\Psi(x^0,y)$ is represented as a linear constraint.
Considering the quadratic loss function $E(\cdot)$ and the affine disjunctive forms of $\Psi(\cdot)$ and $R(\cdot)$, we solve NNRepLayer as a Mixed Integer Quadratic Program (MIQP).
Any feasible solution to the NNRepLayer guarantees that for all input samples $x^0_n$ from $\{(x^0_n,t_n)\}^N_{n=1}$, 
$\Psi(\pi_{\theta_r}(x^0_n),x^0_n)$ is satisfied (for theorems and proofs read Appx. \ref{supp: th}). 

Our technique only ensures the satisfaction of constraints for the repaired samples, so the satisfaction of constraints for the unseen adversarial samples is not theoretically guaranteed. 
To address this problem, we propose Alg. \ref{alg:verifier} that guarantees the satisfaction of constraints $\Psi$ in $\mathcal{X}_r$. 
In this algorithm, NNReplayer is employed with a sound verifier \citep{huang2017safety,katz2017reluplex} in the loop such that our method first returns the repaired network $\pi^{r}_{\theta}$. Then, the verifier evaluates the network. 
If the algorithm terminates, the network is guaranteed to be safe for all other unseen samples in the target input space. Otherwise, the network is not satisfied to be safe and the verifier provides the newly found adversarial samples $\mathcal{X}_r$ for which the guarantees do no hold. In turn, NNRepLayer uses the given samples by the verifier to repair the network. This loop terminates when the verifier confirms the satisfaction of constraints. 

\section{Evaluation}
\begin{figure*}[t]

    \centering
    \includegraphics[width=0.95\textwidth]{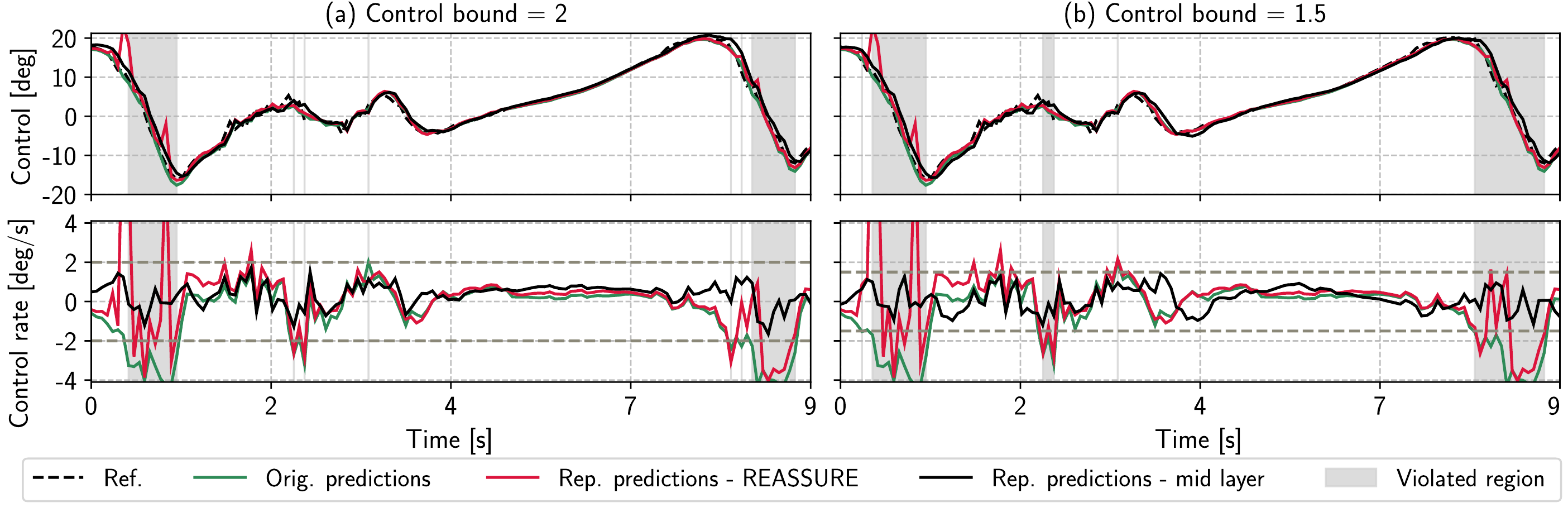}

    \caption{Bounding Ankle angles and Ankle angle rates for bounds (a) $\Delta \alpha_a = 2$ and (b) $\Delta \alpha_a = 1.5$.}
    \label{fig:dynamic_const_xxx}
    \vspace{-5pt}
\end{figure*}
Since DNN may change the control more rapidly than what is feasible for the robotic prosthesis or for the human subject to accommodate, we propose an \textbf{Input-output constraint} over the possible change of control actions from one time-step to the next.
This constraint should act to both smooth the control action in the presence of sensor noise, as well as to reduce hard peaks and oscillations in the control action.
To capture this constraint as an input-output relationship, we trained a three-hidden-layer policy network with $32$ ReLU nodes at each hidden layer.
The network receives the previous $dt$ control actions $\{\alpha_a(i)\}^{t-1}_{i=t-dt}$, and the angle and velocity from the upper and lower limb sensors,
$\alpha_{ul},\dot{\alpha}_{ul}, \alpha_{ll}, \dot{\alpha}_{ll}$ (network inputs $x^0$), respectively. The network then predicts the ankle angle $\alpha_a$ (network output $y$). Using NNRepLayer, we then bound the control rate by applying the constraint 
$\Delta\alpha_a$ by $\lvert \alpha_a(t) - \alpha_a(t-1)\rvert \leq \Delta \alpha^{max}_a$ 
in network repair (repairing a middle layer). 
\begin{wraptable}{r}{7cm}
\vspace{-4pt}
\centering
\caption{Repair Stats*. 
RT: runtime, 
MAE: Mean Absolute Error between the repaired and the original outputs, 
RE: the percentage of adversarial samples that are repaired, 
and 
IB: the percentage of test samples that were originally safe but became faulty after the repair. 
% The metrics are the average of 50 runs. 
% Note that NNRepLayer and REASSURE \citep{FuLi2022iclr} return the same outcomes for the fixed repair samples at each run, 
% so the reported metrics are the average of repairing 50 different trained networks.
}
\label{tab:my-table}

\begin{adjustbox}{width=.5\textwidth}
\setlength{\aboverulesep}{0pt}
\setlength{\belowrulesep}{0pt}
\setlength{\extrarowheight}{.75ex}
% \vspace{2cm}

\begin{tabular}{ccccc}

\toprule
 &RT [s]& MAE & RE [\%] & IB [\%] \\ \midrule
{\cellcolor[HTML]{EFEFEF} NNRepLayer}  
        & \cellcolor[HTML]{EFEFEF} $112\pm122$ & \cellcolor[HTML]{EFEFEF} $0.5\pm 0.03$ & \cellcolor[HTML]{C0C0C0} $98\pm 1$ & \cellcolor[HTML]{C0C0C0} $0.19 \pm 0.18$
        \\
{REASSURE~\citep{FuLi2022iclr}}  
        & $30 \pm 8$ & $0.6\pm 0.03$ & $19\pm 4$ & $85\pm 5$
        \\
{Fine-tune}   
        & $8\pm 2$ & $0.6\pm 0.03$ & $88\pm 2$ & $2.47 \pm 0.49$
        \\ 
{Retrain}  
        & $101 \pm 1$ & $0.5\pm 0.04$ & $98\pm 1$ & $0.28 \pm 0.32$
        \\ 
        \bottomrule 
        \multicolumn{5}{l}{\small *Average of 50 runs.}\\
\end{tabular}%
\end{adjustbox}
\vspace{-5pt}
\end{wraptable}
In our tests, we bounded the control rate by $\Delta \alpha^{max}_a = 1.5$ [deg/s] and $\Delta \alpha^{max}_a = 2$ [deg/s]. 
Our simulation results in Fig. \ref{fig:dynamic_const_xxx} demonstrate that NNRepLayer satisfies both bounds on the control rate which subsequently results in a smoother control output. 
It can also be observed that NNRepLayer successfully preserves the tracking performance of controller. 
 Table \ref{tab:my-table} compares our method with fine-tuning, retraining \citep{sinitsin2019editable,ren2020few,DongEtAl2021qrs,taormina2020performance}, and REASSURE \citep{FuLi2022iclr}. 
As shown in Table \ref{tab:my-table}, retraining and NNRepLayer both perform well in maintaining the 
minimum absolute error and the generalization of constraint satisfaction to the testing samples. 
Comparing to \citep{FuLi2022iclr}, while REASSURE guarantees the satisfaction of constraints in the local faulty linear regions, 
we showed that this method significantly reduces the performance of network in the repaired regions, see Fig.\ref{fig:dynamic_const_xxx}. 
Figure \ref{fig:dynamic_const_xxx} also shows that \citep{FuLi2022iclr} cannot address the given constraints for the faulty samples (introduces almost 500 times more faulty samples compared to our technique).  
Moreover, we tested Alg. \ref{alg:verifier} on a global constraint that ensures $\alpha_{a}$ 
stays within a certain range. 
We used the sound verifier proposed in \citep{TjengXT2019iclr} for the verification. 
To evaluate the algorithm, we used the adversarial accuracy metric ($ACC_{\epsilon}$). 
For a given adversarial sample set $X_{adv}\subseteq X_r$, 
this metric evaluates the portion of samples that are robust to perturbations in the $l_{\infty}$ ball of radius $\epsilon$ around each sample. Our method results $ACC_{0.5} = 100\%$ after one round of repair, $ACC_{1}=88\%$ after the second iteration, and $ACC_{1}=100\%$ after the third iteration.
Finally, the largest network that we successfully repaired had 256 neurons in each hidden layer that took up to 10 hours.  
Similar network structure and sizes are frequently used in robotics and control tasks for example researchers in Google Brain trained a robot locomotion task using a network with 2 hidden layers and 256 nodes \citep{haarnoja2019learning}. 
For details on the experimental setup and other results, read Appx. \ref{supp:result}.

\section{Conclusion}
\label{sec:conclusion}
We introduced a framework for training neural network controllers that certifiably satisfy a formal set of safety constraints. 
Our approach, NNRepLayer, performs a global optimization step in order to perform layer-wise repair of neural network weights tested on a lower-leg prosthesis satisfying a variety of constraints. 
We argue that this type of approach is critical for human-centric and safety-critical applications of robot learning, e.g., the next-generation of assistive robotics.

\section*{Acknowledgment}
This work was partially supported by the National Science Foundation under grants CNS-1932068, IIS-1749783, and CNS-1932189.

\bibliographystyle{unsrtnat}
\bibliography{references}

\newpage
% \section*{Appendix}
\appendix
\section{Interval Arithmetic Method}
\label{supp: IA}
To illustrate how we generated a tight valid bound for each ReLU activation node, we used the Interval Arithmetic method \citep{moore2009introduction, TjengXT2019iclr}. 
Interval arithmetic is widely used in verification to find an upper and a lower bounds over the relaxed ReLU activations given a bounded set of inputs. 
We used the same approach to find the tight bounds over the ReLU nodes assuming the weights can only perturb inside a bounded $l_{\infty}$ error with respect to the original weights. 
Assume we denote each input variable of repair layer $L$ as $x^{L-1}(i)$, the weight term that connect $x^{L-1}(i)$ to $x^{L}(j)$ as $\theta^L_w(ij)$, and the bias term of nodes $x^L(j)$ as $\theta^L_b(j)$. 
Given the bounds for variables $\theta^L_w(ij)\in \big[\underline{\theta}^l_w(ij), \bar{\theta}^L_w(ij)\big]$ 
and $\theta^L_b(ij)\in \big[\underline{\theta}^L_b(ij), \bar{\theta}^L_b(ij)\big]$, the interval arithmetic gives the valid upper and lower bounds for $x^{L}(j)$ as
\begin{align*}
&\bar{x}^{L}(j) = \sum_{i} \Big(\bar{\theta}^L_w(ij)\max(0,x^{L-1}(i)) + \underline{\theta}^L_w(ij)\min(0,x^{L-1}(i))\Big) + \bar{\theta}^L_b(ij),\text{ and}\\
&\underline{x}^{L}(j) = \sum_{i} \Big(\underline{\theta}^L_w(ij)\max(0,x^{L-1}(i)) + \bar{\theta}^L_w(ij)\min(0,x^{L-1}(i))\Big) + \underline{\theta}^L_b(ij),
\end{align*}
respectively. The bounds over the ReLU nodes in the subsequent layers $l = L+1,\cdots N$ are obtained as 
\begin{align*}
&\bar{x}^{l}(j) = \sum_{i} \Big(\bar{x}^{l-1}\max(0,\theta^l_w(ij)) + \underline{x}^{l-1}\min(0,\theta^l_w(ij))\Big) + \theta^l_b(ij),\\
&\underline{x}^{l}(j) = \sum_{i} \Big(\underline{x}^{l-1}\max(0,\theta^l_w(ij)) + \bar{x}^{l-1}\min(0,\theta^l_w(ij))\Big) + \theta^l_b(ij).
\end{align*}

% \newpage
\section{Theorems and Proofs}
\label{supp: th}
\begin{theorem}[Soundness of NNRepLayer]
\label{theorem: feasibility}
Given the predicate $\Psi(y,x^0)$, 
and the input-output data set $\{(x^0_n,t_n)\}^N_{n=1}$ 
sampled from  $(\mathcal{X}\times\mathcal{T}) \cup (\mathcal{X}_r\times\Tilde{\mathcal{T}})$ 
over the sets $\mathcal{X}$, $\mathcal{X}_r$, $\mathcal{T}$, 
and $\Tilde{\mathcal{T}}$, 
assume that $\theta_w^{l}$ and $\theta_b^{l}$ are feasible solutions to (\ref{eq:cost})-(\ref{eq:w_bound}).
Then, $\Psi(\pi_{\theta_r}(x^0_n),x^0_n)$ is satisfied
for all input samples $x^0_n$.
% The repaired network $\pi_{\theta}^{\theta_w^{l} ,\theta_b^{l}}$ is guaranteed to satisfy 
% $\Psi(y,x^0)$ for the input-output data set $\{(x^0_n,t_n)\}^N_{n=1}$.
\end{theorem}
\begin{proof}
 Since the feasible solutions $\theta_w^{l}$ and $\theta_b^{l}$ satisfy the hard constraint (\ref{eq:predicates}) 
 for the repair data set $\{(x^0_n,t_n)\}^N_{n=1}$, 
 $\Psi(\pi_{\theta_r}(x^0_n),x^0_n)$ is satisfied.
%  $\pi_{\theta_r}$ always satisfies 
% $\Psi(y,x^0)$ 
% for the input-output data set $\{(x^0_n,t_n)\}^N_{n=1}$.
\end{proof}

Given Thm. \ref{theorem: feasibility}, the following Corollary is straightforward.
\begin{corollary}
Given the predicate $\Psi(y,x^0)$, 
and the input-output data set $\{(x^0_n,t_n)\}^N_{n=1}$ 
sampled from  $(\mathcal{X}\times\mathcal{T}) \cup (\mathcal{X}_r\times\Tilde{\mathcal{T}})$ 
over the sets $\mathcal{X}$, $\mathcal{X}_r$, $\mathcal{T}$, 
and $\Tilde{\mathcal{T}}$, 
assume that $\theta_w^{l^*}$ and $\theta_b^{l^*}$ are the optimal solutions to the NNRepLayer (\ref{eq:cost})-(\ref{eq:w_bound}).
Then, for all input samples $x^0_n$ from $\{(x^0_n,t_n)\}^N_{n=1}$, 
$\Psi(\pi_{\theta_r}(x^0_n),x^0_n)$ is satisfied.
\end{corollary}

\begin{theorem}[Soundness of Alg. \ref{alg:verifier}]
\label{theorem: soundness}
Assume $\textsc{Verifier}()$ is a sound verifier. If the Alg. \ref{alg:verifier} terminates, the predicate $\Psi$ is satisfied by the repaired network $\pi^{r}_{\theta}$.
\end{theorem}
\begin{proof}
    Given that $\textsc{Verifier}()$ is assumed to be sound, if the algorithm terminates, $\mathcal{X}_{r}$ is empty which means $\textsc{Verifier}()$ did not find other samples that violate $\Psi$. Therefore,  the predicate $\Psi$ is guaranteed to be satisfied by $\pi^{r}_{\theta}$.
\end{proof}

% \newpage
\section{More Details on Experimental Results}
\label{supp:result}

\subsection{Experimental Setup}
We trained a policy network $\pi_{\theta}$ for controlling a prosthesis, 
which then undergoes the repair process to ensure compliance with the safety constraints. 
% \begin{wrapfigure}[12]{r}{0.2\textwidth}
% \vspace{-15pt}
%   \begin{center}
%     \includegraphics[scale = 1.1]{prosthesis.pdf}
%   \end{center}
%   \caption{Prosthetic device model.}
%   \label{fig:foot}
% \end{wrapfigure}
To this end, we first train the model using an imitation learning~\citep{schaal1999imitation} strategy. 
For data collection, we conducted a study approved by the Institutional Review Board (IRB), 
in which we recorded the walking gait of a healthy subject without any prosthesis. 
Walking data included three inertial measurement units (IMUs) mounted via straps to the upper leg (Femur), 
lower leg (Shin), and foot. 
The IMUs acquired both the angle and angular velocity of each limb portion in the world coordinate frame at 100Hz. 
Ankle angle $\alpha_a$ was calculated as a post process from the foot and lower limb IMUs. 
We then trained the NN to generate the ankle angle from upper and lower limb IMU sensor values.
% More specifically, the NN model receives the angle and velocity from the upper and lower limb sensors 
% (network inputs $x^0$),
% $\alpha_{ul},\dot{\alpha}_{ul}, \alpha_{ll}, \dot{\alpha}_{ll}$, respectively,
% to predict the ankle angle $\alpha_a$ (network output $y$) which is, later, 
% used as the control parameter for a PD controller on the prosthetic. 
% See Fig.~\ref{fig:foot} for a visualization of the individual sensor readings. 
We used a sliding window of input variables, denoted as $dt$ ($dt=10$ in all our experiments), 
to account for the temporal influence on the control parameter and 
to accommodate for noise in the sensor readings. 
Therefore, the input to the network is $dt\times \lvert x^0\rvert$,  or more specifically the current and previous $dt$ sensor readings. 
% In all experiments, we trained a three-hidden-layer deep policy network with $32$ ReLU nodes at each hidden layer.
After the networks were fully trained we assessed the policy for constraint violations and collected samples for NNRepLayer.
\begin{wrapfigure}{r}{0.5\textwidth}
\vspace{-5pt}
% \vspace{-5pt}
    \centering \includegraphics[scale=0.48]{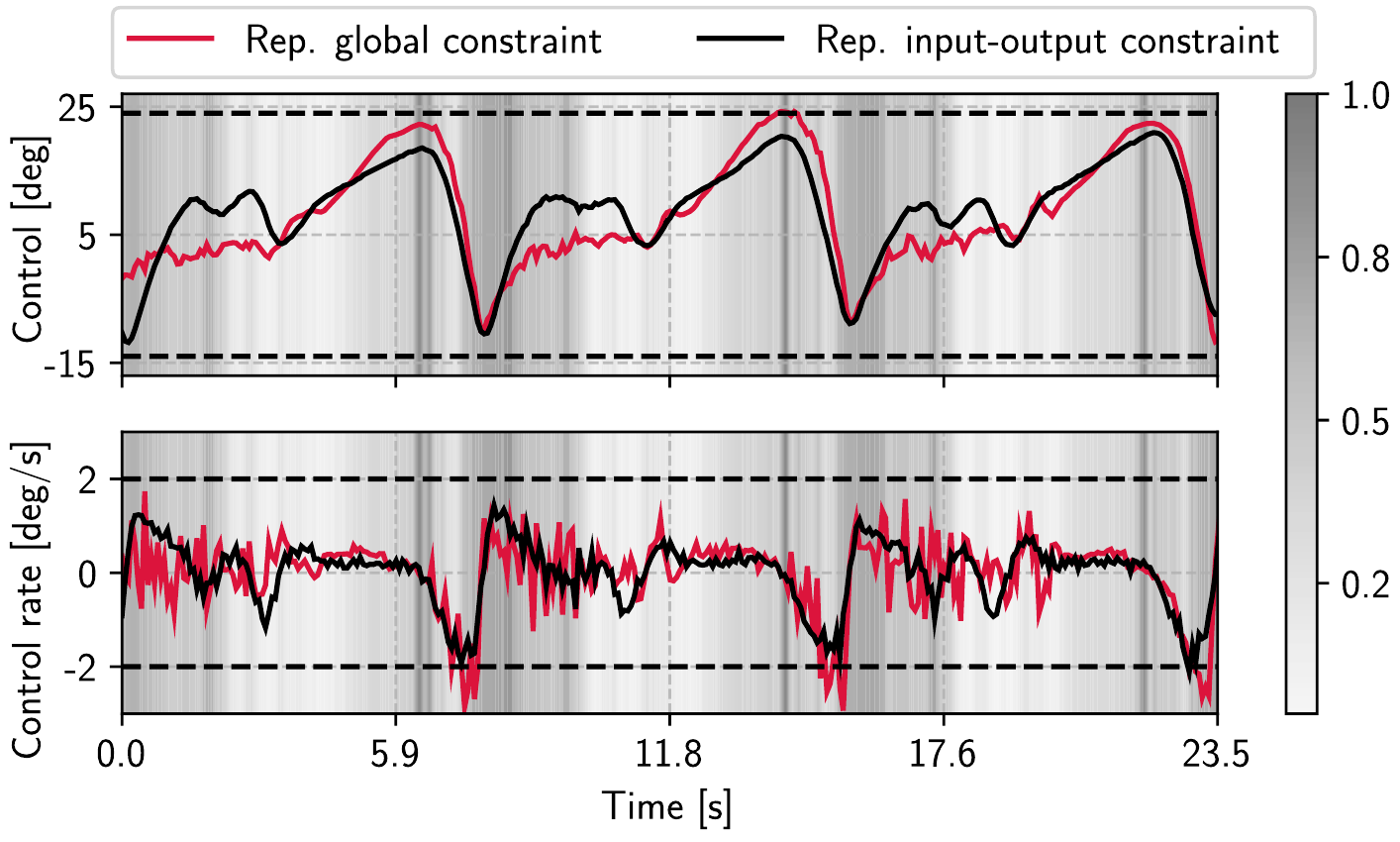}
%   \vspace{-10pt}
  \caption{Real prosthesis walking test results for imposing the global constraint of $[-14,24]$ to the control (shown in red) and bounding the control rate by $2$ [deg/s] (shown in black). The color bar represents the normalized $L_2$-distance of each test input to its nearest neighbor in the repair set. }
  \vspace{-5pt}
\label{fig:exec_signal}
\end{wrapfigure}
We tested NNRepLayer on the last and the second to the last layer of network policy
to satisfy the constraints with a subset of the original training data including both adversarial and non-adversarial samples. 
In all experiments, we used $150$ samples in NNRepLayer and a held out set of size $2000$ for testing. 
Finally, the repaired policies to satisfy global and input-output constraints are tested on a prosthetic device for $10$ minutes of walking, see Fig. \ref{fig:exec_signal}. 
More specifically, the same healthy subject was fitted with an ankle bypass; a carbon fiber structure molded to the lower limb and constructed such that a prosthetic ankle can be attached to allow the able-bodied subject to walk on the prosthesis. 
The extra weight and off-axis positioning of the device incline the individual towards slower, asymmetrical gaits that generates strides out of the original training distribution \citep{cortino2022stair, gao2020recurrent}. 
The participant is then asked to walk again for $10$ minutes to assess whether constraints are satisfied.
Adversarial samples in the repair data set are hand-labeled for fine-tuning and retraining
so that the target outputs satisfy the given predicates.
In fine-tuning, as proposed in \citep{sinitsin2019editable,taormina2020performance}, we used the collected adversarial data set to train all the parameters of the original policy by gradient descent using a small learning rate ($10^{-4}$). To avoid over-fitting to the adversarial data set, we trained the weights of the top layer first, and thereafter fine-tuned the remaining layers for a few epochs.
The same hand-labeling strategy is applied in retraining, except that a new policy is trained from scratch for all original training samples.
In both methods, we trained the policy until all the adversarial samples in 
the repair data set satisfy the given predicates on the network's output. Our code is available on GitHub: \url{https://github.com/k1majd/NNRepLayer.git}.
% \newpage
\subsection{Experiments and Results}
\paragraph{Global Constraint.} 
% \begin{figure*}[t]

%     \centering
%     \includegraphics[width=\textwidth]{global_constraint.pdf}

%     \caption{Global constraint: (a) ankle angle, $\alpha_{a}$, (b) the error between the predicted and the reference controls, (c) the violation degree vs. $L_2$-distance between the test and repair sample inputs.}
%     \label{fig:global_const}
% \end{figure*}
\begin{figure*}[t]

    \centering
    \includegraphics[width=0.95\textwidth]{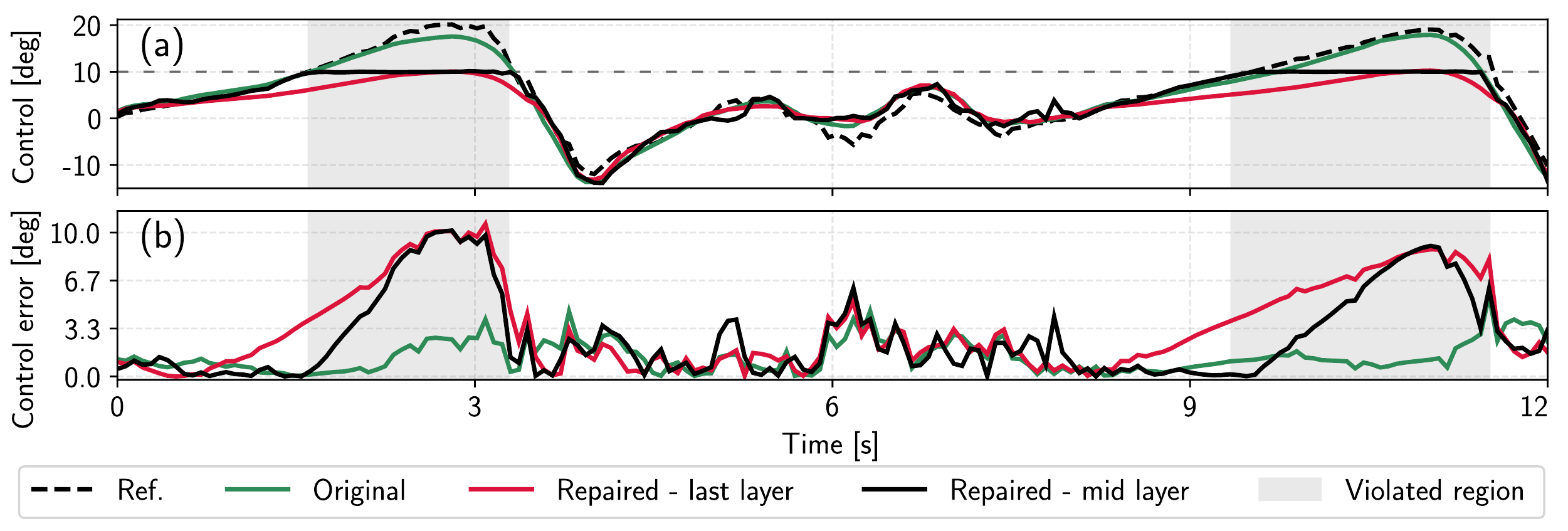}

    \caption{Global constraint: (a) ankle angle, $\alpha_{a}$, (b) the error between the predicted and the reference controls.}
    \label{fig:global_const}
\end{figure*}

The global constraint ensures that the prosthesis control, i.e., $\alpha_{a}$, stays within a certain range and never outputs an unexpected large value that disturbs the user's walking balance. 
Additionally, the prosthetic device we utilized in these scenarios contains a parallel compliant mechanism.
As such, either the human subject or the robotic controller could potentially drive the mechanism into the hard limits, potentially damaging the device.
In our walking tests, see Fig. \ref{fig:exec_signal}, we therefore specified global constraints such that the ankle angle stays within the bounds of $[-14,24]$ [deg] regardless of whether it is driven by the human or the robot. 
In simulation experiments. see Fig. \ref{fig:global_const} (a), we enforced artificially strict bounds on the ankle angle $\alpha_{a}$ 
to never exceed $\alpha_{a}=10$ [deg] which is a harder bound to satisfy.
% We defined the degree of violation as $0$ if 
% $\alpha_{a} \in [\alpha_a^{min}, \alpha_a^{max}]$, and $\min\{\lvert \alpha_{a}-\alpha_a^{max}\rvert , \lvert \alpha_{a} - \alpha_a^{min}\rvert\}$, otherwise.

% As shown in Fig. \ref{fig:global_const} (a), 
% the repaired network successfully satisfies
% the constraints in the original faulty regions
% while maintaining the tracking performance of the controller in the unconstrained regions. 

\paragraph{Conditional Constraint.} 
Depending on the ergonomic needs and medical history of a patient,
the attending orthopedic doctor or prosthetist may identify certain body configurations that are harmful, e.g., they may increase the risk of osteoarthritis or musculoskeletal conditions~\citep{morgenroth2012osteoarthritis, ekizos2018}. Following this rationale, we define a region $\mathcal{S}$ 
of joint angles space that should be avoided.
An example of such a region is demonstrated in Fig.~\ref{fig:flow_fig} as 
a grey box 
$\mathcal{S}=\{(\alpha_{ul}, \alpha_{a})~|~\alpha_{ul}\in [-2, -0.5], \alpha_{a}\in[1,3]\}$ 
in the joint space of ankle and femur angles. 
\begin{wrapfigure}{r}{0.355\textwidth}
\vspace{-5pt}
 \centering \includegraphics[scale=0.7]{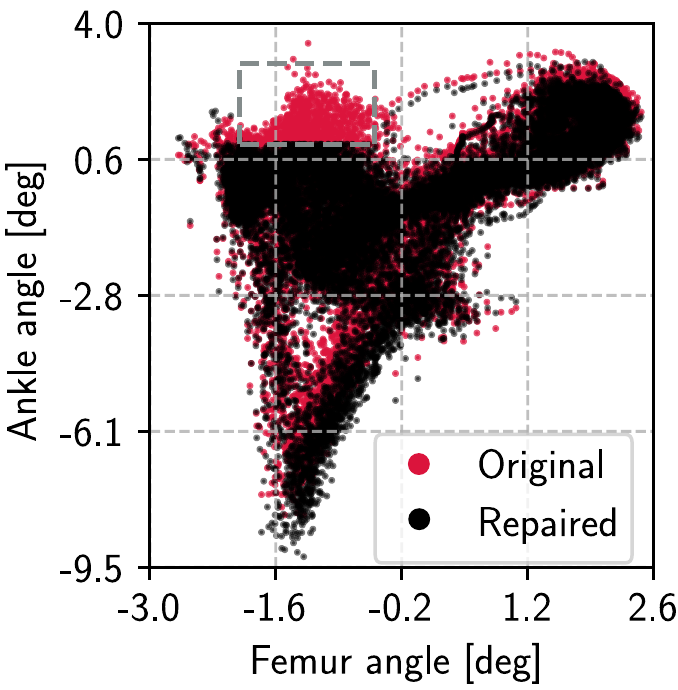}
  \caption{Enforcing the conditional constraints to keep the joint femur-ankle angles out of the grey box.}
  \label{fig:flow_fig}
  \vspace{-10pt}
\end{wrapfigure}
To satisfy this constraint the control rate should be tuned such that the joint ankle and femur angles stay out of set $\mathcal{S}$. 
This constraint can be defined as an if-then-else proposition  $\alpha_{ul}\in[-2, -0.5] \implies \big(\alpha_a\in(-\infty,1]\big)\vee\big(\alpha_a\in[3,\infty)\big)$ 
which can be formulated as the disjunction of linear inequalities on the network's output. 
% For each given test input and its corresponding output $\alpha_a$, 
% the degree of violation is defined as the distance of $\alpha_a$ to the box if $\alpha_a$ is outside the box, and $0$
%  otherwise.
%  as $0$ if $\alpha_a$ is 
% inside the box and its distance to the box otherwise.
Figure \ref{fig:flow_fig} demonstrates the output of new policy after repairing with NNRepLayer. As it is shown, our method avoids the joint ankle and femur angles to enter the unsafe region $S$. 
% Figure \ref{fig:violate_dyn_cond} (b) also illustrates low output violation degree for the distant test input samples from the repair input set.
Finally, we observed that repairing the last layer does not result in a feasible solution.
% \paragraph{Comparison w\textbackslash Fine-tuning, Retraining.}
% In each experiment, we demonstrated the violation degree of our method compared with fine-tuning and retraining. Table \ref{tab:my-table} better illustrates the success of our method in satisfying the constraints while maintaining the control performance.
% \begin{wrapfigure}{r}{0.4\textwidth}
%  \centering \includegraphics[scale=0.7]{flow_fig.pdf}
%   \caption{NNRepLayer enforcing the conditional constraints to keep the joint femur-ankle angles out of the grey box.}
%   \label{fig:flow_fig}
%   \vspace{-30pt}
% \end{wrapfigure}
% As shown in Table \ref{tab:my-table}, retraining and NNRepLayer both perform well in maintaining the 
% minimum absolute error and the generalization of constraint satisfaction to the unseen testing samples 
% for global and input-output constraints. 
% However, according the Table \ref{tab:my-table} the satisfaction of if-then-else constraints is challenging for retraining and fine-tuning as the repair efficacy is dropped by almost $30\%$ using these techniques. It also highlights the power of our technique in generalizing the satisfaction of conditional constraints to the unseen cases.
\paragraph{Comparison w\textbackslash Fine-tuning, Retraining, and REASSURE.}
% In each experiment, we demonstrated the violation degree and the control signals of our method compared with fine-tuning, retraining \citep{sinitsin2020editable,taormina2020performance,ren2020few,DongEtAl2021qrs}, and REASSURE \citep{FuLi2022iclr}. 
% Comparing to \citep{FuLi2022iclr}, while REASSURE guarantees the satisfaction of constraints in the local repaired linear regions, 
% we showed that this method significantly reduces the performance of network in the repaired local regions, see Figures \ref{fig:global_const} and \ref{fig:dynamic_const_xxx}. 
% This method cannot address the input-output constraints given the faulty samples, and it introduces 500 times more faulty samples compared to our technique. 
% REASSURE cannot also accommodate the conditional constraints.
% Unlike REASSURE that guarantees the satisfaction of constraints for the samples in the same linear region as the repaired samples, 
% our technique only guarantees the satisfaction of constraints for the repaired samples. 
% While we empirically showed the  generalizability of our technique in a local neighborhood of the repaired samples, 
% our method does not theoretically guarantee the satisfaction of constraints for the unseen adversarial samples.
% We proposed a sound algorithm in the supplementary materials, 
% Sec. \ref{supp: soundness}, that  guarantees the safety for all other unseen samples. 
% We also showed that this method performs poorly in generalizing the repair for the input-output and conditional constraints.
Table \ref{tab:my-table} better illustrates the success of our method in satisfying the constraints while maintaining the control performance.
As shown in Table \ref{tab:my-table}, retraining and NNRepLayer both perform well in maintaining the 
minimum absolute error and the generalization of constraint satisfaction to the unseen testing samples 
for global and input-output constraints. 
However, the satisfaction of if-then-else constraints is challenging for retraining and fine-tuning as the repair efficacy is dropped by almost $30\%$ using these techniques. 
It also highlights the power of our technique in generalizing the satisfaction of conditional constraints to the unseen cases.

% \begin{table}[t]
% \centering
% \caption{The table reports the Mean Absolute Error (MAE) between the repaired and the original outputs, the percentage of adversarial samples that are repaired, also called the Repair Efficacy, and the runtime of NNRepLayer for each constraint case. The metrics for fine-tuning and retraining are the average of 50 runs. For fixed repair samples, NNRepLayer returns the same outcome each run.}
% \label{tab:my-table}

% \begin{adjustbox}{width=1\textwidth,center=\textwidth}
% \resizebox{\textwidth}{!}{%
% \setlength{\aboverulesep}{0pt}
% \setlength{\belowrulesep}{0pt}
% \setlength{\extrarowheight}{.75ex}
% \begin{tabular}{@{}cccc
% >{\columncolor[HTML]{EFEFEF}}c 
% >{\columncolor[HTML]{EFEFEF}}c 
% >{\columncolor[HTML]{C0C0C0}}c ccc@{}}
% \toprule
%  & \multicolumn{3}{c}{MAE}      & \multicolumn{3}{c}{\cellcolor[HTML]{EFEFEF} Repair Efficacy [\%]} & \multicolumn{3}{c}{Runtime {[}s{]}} \\ 
%  & Fine-tune & Retrain & NNRepLayer & Fine-tune            & Retrain           & NNRepLayer           & Fine-tune    & Retrain   & NNRepLayer   \\ \midrule
% Global       & $1.2\pm0.03$ & $1.4\pm0.08$           & $1.4\pm 0.11$   & $97\pm4$ & $98\pm3$ & $99\pm 1$ & $25\pm13$ & $127\pm30$ & $233 \pm 159$ \\
% Input-output & $0.6\pm0.03$          & $0.5\pm0.04$  & $0.5\pm 0.03$ & $88\pm2$ & $98\pm1$ & $98\pm1$ & $8\pm2$   & $101\pm1$  & $112 \pm 122$  \\
% Conditional  & $0.7\pm0.10$          & $0.31\pm0.03$ & $0.35 \pm 0.07$         & $72\pm5$ & $76\pm2$ & $93\pm 2$ & $18\pm3$  & $180\pm2$  & $480 \pm 110$ \\ \bottomrule 
% \end{tabular}%
% }
% \end{adjustbox}
% \end{table}

\begin{table*}[t]
\centering
\caption{The table reports: 
RT: runtime, 
MAE: Mean Absolute Error between the repaired and the original outputs, 
RE: the percentage of adversarial samples that are repaired (Repair Efficacy), 
and 
IB: the percentage of test samples that were originally safe but became faulty after the repair (Introduced Bugs). 
The metrics are the average of 50 runs. 
% Note that NNRepLayer and REASSURE \citep{FuLi2022iclr} return the same outcomes for the fixed repair samples at each run, 
% so the reported metrics are the average of repairing 50 different trained networks.
}
\label{tab:my-table}

\begin{adjustbox}{width=1\textwidth,center=\textwidth}
\setlength{\aboverulesep}{0pt}
\setlength{\belowrulesep}{0pt}
\setlength{\extrarowheight}{.75ex}
% \vspace{2cm}

\begin{tabular}{@{}
ccccccccc
@{}}
\toprule
 & \multicolumn{4}{c}{\cellcolor[HTML]{EFEFEF} NNRepLayer} 
 & \multicolumn{4}{c}{REASSURE~\citep{FuLi2022iclr}} \\ 
    &\cellcolor[HTML]{EFEFEF} RT [s] & \cellcolor[HTML]{EFEFEF} MAE & \cellcolor[HTML]{C0C0C0} RE [\%] & \cellcolor[HTML]{C0C0C0} IB [\%] 
    &RT [s]& MAE & RE [\%] & IB [\%]  \\ \midrule
Global  
        & \cellcolor[HTML]{EFEFEF} $233\pm159$ & \cellcolor[HTML]{EFEFEF} $1.4\pm 0.11$ & \cellcolor[HTML]{C0C0C0} $99\pm 1$ & \cellcolor[HTML]{C0C0C0} $0.09 \pm 0.20$
        & $14 \pm 1$ & $2.3\pm 0.78$ & $97\pm 1$ & $0$
        \\
Input-output 
        & \cellcolor[HTML]{EFEFEF} $112\pm122$ & \cellcolor[HTML]{EFEFEF} $0.5\pm 0.03$ & \cellcolor[HTML]{C0C0C0} $98\pm 1$ & \cellcolor[HTML]{C0C0C0} $0.19 \pm 0.18$
        & $30 \pm 8$ & $0.6\pm 0.03$ & $19\pm 4$ & $85\pm 5$
        \\
Conditional  
        & \cellcolor[HTML]{EFEFEF} $480\pm110$ & \cellcolor[HTML]{EFEFEF} $0.35\pm 0.07$ & \cellcolor[HTML]{C0C0C0} $93\pm 2$ & \cellcolor[HTML]{C0C0C0} $0.11 \pm 0.26$
        & Infeasible & Infeasible & Infeasible & Infeasible
        \\ 
        % \bottomrule \\
        \toprule
& \multicolumn{4}{c}{Fine-tune} 
 & \multicolumn{4}{c}{Retrain} \\ 

    &RT [s]& MAE & RE [\%] & IB [\%]  
    &RT [s]& MAE & RE [\%] & IB [\%]  \\ \midrule
Global  
        & $25\pm13$ & $1.2\pm 0.03$ & $97\pm 4$ & $0.95 \pm 0.45$
        & $127 \pm 30$ & $1.4\pm 0.08$ & $98\pm 3$ & $0.65 \pm 0.40$
        \\
Input-output 
        & $8\pm 2$ & $0.6\pm 0.03$ & $88\pm 2$ & $2.47 \pm 0.49$
        & $101 \pm 1$ & $0.5\pm 0.04$ & $98\pm 1$ & $0.28 \pm 0.32$
        \\
Conditional  
        & $18\pm 3$ & $0.7\pm 0.10$ & $72\pm 5$ & $0.27 \pm 0.25$
        & $180 \pm 2$ & $0.31\pm 0.03$ & $76\pm 2$ & $0.12 \pm 0.35$
        \\ \bottomrule 
\end{tabular}%
\end{adjustbox}
\end{table*}

\subsection{Testing Repair on Larger Networks}
\label{supp:256}
To demonstrate the scalability of our method, we conducted a repair experiment on a network with 256 neurons in each hidden layer. We used 1000  samples for repair and 2000 samples for testing. We formulate this problem in MIQP and run the program on a Gurobi \citep{gurobi} solver. We terminated the solver after 10 hours and report the best found feasible solution. Figure \ref{fig:net_256} and Table \ref{table:net_256} show the control signal, and the statistical results of this experiment, respectively. As demonstrated, our technique repaired a network with up to 256 nodes with 100\% repair efficacy in 10 hours.  Similar network structure and sizes are frequently used in robotics and controls tasks. Examples include \citep{fernandez2020deep} (3 hidden layer, 256 nodes), \citep{landgraf2021reinforcement} (2 hidden layer, 64 nodes), \citep{pinosky2022hybrid} (2 hidden layer, 200 nodes), \citep{zimmer2018developmental} (2 hidden layer, 50 nodes), and \citep{kristoffersen2021user} (2 hidden layer, 50 nodes).
\begin{table}[tbh]
\centering
\caption{ Experimental results for repairing a network with 256 nodes in each hidden layer for the input-output constraint repair, maximum ankle angle rate of $2$ [rad/s]. The table reports the size of network, the number of samples, the Mean Absolute Error (MAE) between the repaired and the original outputs, the percentage of adversarial samples that are repaired (Repair Efficacy), and the runtime.}
\label{table:net_256}
\begin{adjustbox}{width=0.9\textwidth,center=\textwidth}
\begin{tabular}{@{}cccccc}
\toprule
  & Network Size      & Number of Samples & MAE & Repair Efficacy [\%] & Runtime [h] \\ \midrule
  Input-output Constraint & 256 & 1000 & 0.62 & 100 & 10\\ \bottomrule 
\end{tabular}%
\end{adjustbox}
\end{table}
%===============================================================================
\begin{figure*}[tbh]

    \centering
    \includegraphics[width=\textwidth]{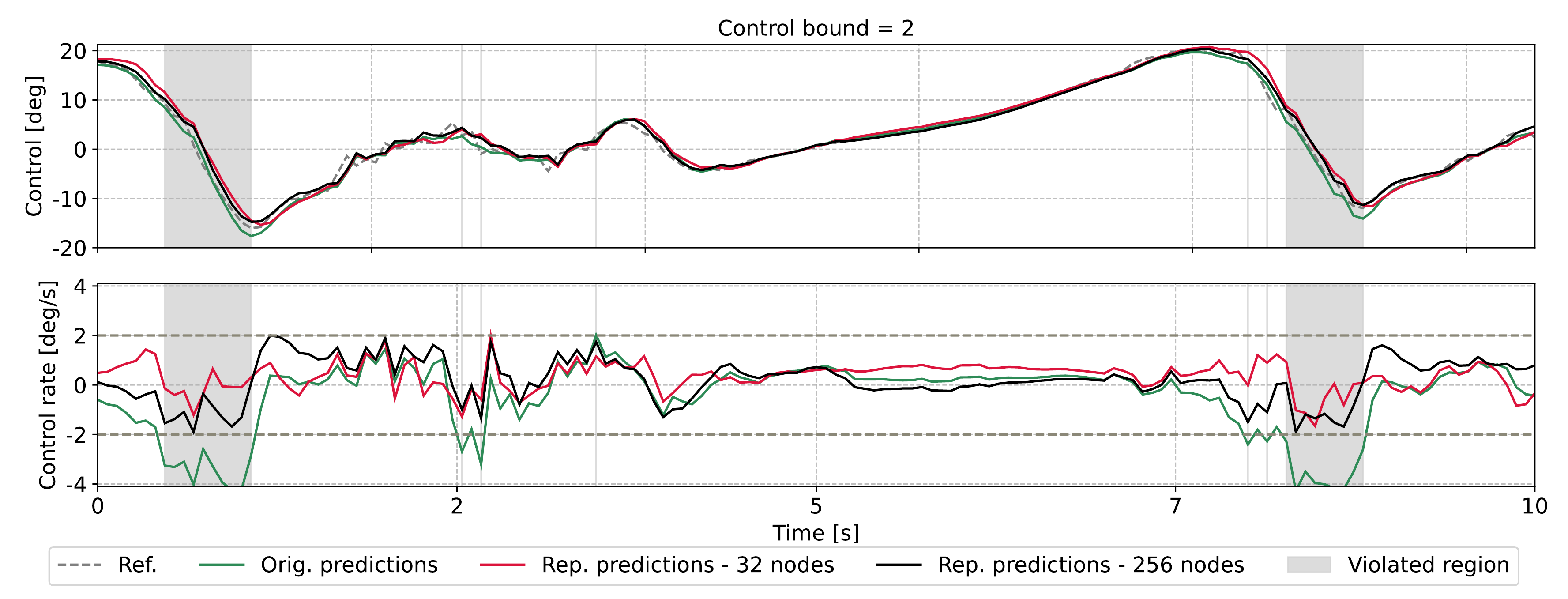}

    \caption{Input-output constraints for networks with 32 (red) and 256 (black) nodes in each hidden layer: Ankle angles and Ankle angle rates for bounds $\Delta \alpha_a = 2$.}
    \label{fig:net_256}
\end{figure*}

\subsection{Heuristics for Computational Speedups of NNRepLayer}
\label{supp: subnode}
In this section we provide examples for computational speedups and heuristics that lead to faster neural network repair. 
To this end, we repaired randomly selected nodes of a single hidden layer in a network with 64 hidden nodes for 35 times. 
We let the solver run for 30 minutes in each experiment (versus the full repair that is solved in 6 hours). 
Figure \ref{fig:random_nodes} demonstrates the mean absolute error (MAE), the total number of repaired weights, repair efficacy, and the original MIQP cost. 
Here, to detect the sparse nodes that can satisfy the constraints, we solved the original full repair by adding the $l_{1}$ norm-bounded error of repaired weights with respect to their original values to the MIQP cost function. 
The bold bars in Fig. \ref{fig:random_nodes} demonstrate the results of repairing the $10$ randomly-selected sparse nodes. 
Repairing of the obtained sparse nodes reached a cost value very close to the cost value of the originally full repair problem ($42.99$ versus $40.29$), \textbf{in only 30 minutes} versus 6 hours. 
As illustrated, repair of some random nodes also results in infeasibility (blank bars) that shows these nodes cannot repair the network. 
In our future work, we aim to explore the techniques such as neural network pruning \citep{hassibi1993optimal, lecun1989optimal}
to select and repair just the layers and nodes that can satisfy the constraints instead of repairing a full layer. 
Our results in this experiment show that repairing partial nodes can significantly decrease the computational time of our technique.
\begin{figure*}[t]

    \centering
    \includegraphics[width=\textwidth]{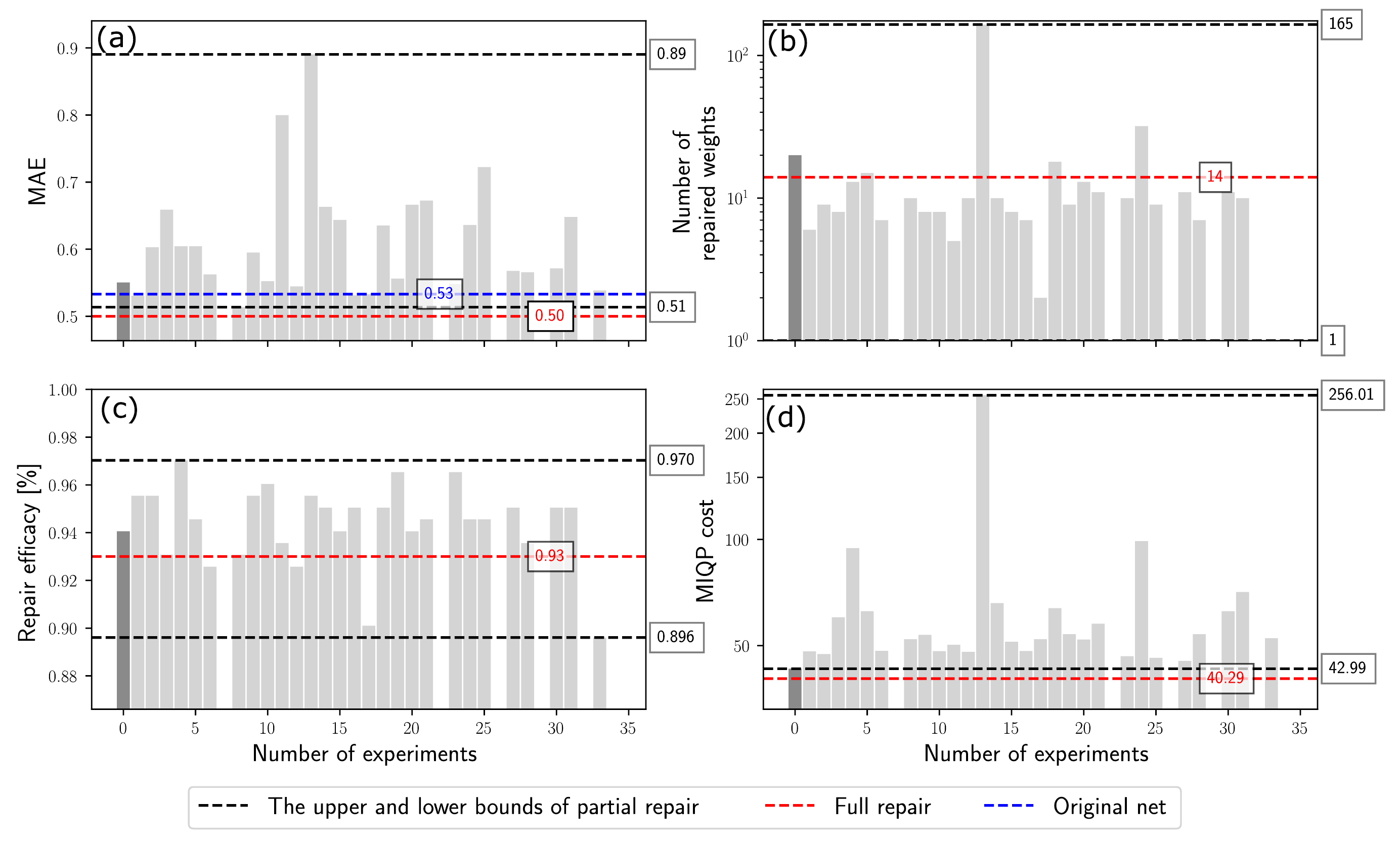}

    \caption{Partial repair versus full repair: we repaired 10 randomly selected nodes for 30 minutes in a network with 64nodes in each hidden layer (a) mean absolute error (MAE), (b) the total number of repaired weights, (c) repair efficacy, and (d) MIQP cost. We performed this random selections for 35 times.}
    \label{fig:random_nodes}
\end{figure*}

\end{document}